\documentclass[10pt,reqno]{amsart}
\usepackage[square,sort,comma,numbers]{natbib}
\usepackage{amsaddr}
\usepackage{etoolbox}

\patchcmd{\section}{\scshape}{\bfseries\scshape}{}{}
\makeatletter
\patchcmd{\maketitle}{\@fnsymbol}{\@alph}{}{}  
\renewcommand{\@secnumfont}{\bfseries}
\makeatother

\input{our_macros.sty}

\title{SiMCa: Sinkhorn Matrix Factorization with Capacity Constraints}

\author{Eric Daoud$^1$, Luca Ganassali$^2$, Antoine Baker$^2$ \and Marc Lelarge$^2$\\ \vspace{0.5cm}
\emph{$^1$Inria, DI/ENS, Institut Curie, Paris, France\\%
    $^2$Inria, DI/ENS, Paris, France}
}

\begin{document}


\begin{abstract}
    For a very broad range of problems, recommendation algorithms have been increasingly used over the past decade. In most of these algorithms, the predictions are built upon user-item affinity scores which are obtained from  high-dimensional embeddings of items and users. In more complex scenarios, with geometrical or capacity constraints, prediction based on embeddings may not be sufficient and some additional features should be considered in the design of the algorithm. 
    
    In this work, we study the recommendation problem in the setting where affinities between users and items are based both on their embeddings in a latent space and on their geographical distance in their underlying euclidean space (e.g., $\mathbb{R}^2$), together with item capacity constraints. This framework is motivated by some real-world applications, for instance in healthcare: the task is to recommend hospitals to patients based on their location, pathology, and hospital capacities. In these applications, there is somewhat of an asymmetry between users and items: items are viewed as static points, their embeddings, capacities and locations constraining the allocation. Upon the observation of an optimal allocation, user embeddings, items capacities, and their positions in their underlying euclidean space, our aim is to recover item embeddings in the latent space; doing so, we are then able to use this estimate e.g. in order to predict future allocations.
    
    We propose an algorithm (SiMCa) based on matrix factorization enhanced with optimal transport steps to model user-item affinities and learn item embeddings from observed data. We then illustrate and discuss the results of such an approach for hospital recommendation on synthetic data.
\end{abstract}

\maketitle


\section{Introduction}

In a very broad range of applications -- many of them being led by e-commerce leaders (Amazon \cite{linden_amazoncom_2003}, Netflix \cite{koren_matrix_2009}) -- recommendation algorithms have been increasingly used over the past decade. These algorithms are capable of showing users a personalized selection of items they may like, based on their interests and user behavior.

Up to now, the predictions are built upon user-item affinity scores (e.g., user/movie ratings) which are obtained from high-dimensional embeddings of items and users. While these approaches work for most e-commerce applications, there are other natural settings in which more attributes should be considered in the recommendation process. For instance, item capacity constraints are of paramount importance in location or route recommendation, where recommending the same item to every user could lead to congestion and significantly deteriorate  user experience \cite{christakopoulou_recommendation_2017}. Moreover, in the case of location recommendation, travel distance is also a key factor: the user's choice is often the result of a trade-off between affinity and proximity \cite{zhao_survey_2016}. In the healthcare sector, patients are usually addressed to an hospital by their general practitioner -- or by word of mouth. Since the choice of hospital and practitioner may be critical, an important issue is to make sure that patients are routed to the best place possible -- namely to a nearby and adapted structure, without capacity saturation. 

In this work, we study the recommendation problem in the setting where affinities between users and items are based both on their embeddings embeddings in a latent space and on their geographical distance in their underlying euclidean space (e.g., $\mathbb{R}^2$), together with item capacity constraints.  Upon the observation of an optimal allocation, user embeddings, items capacities, and their positions in the euclidean space, our aim is to recover item embeddings in the latent space; doing so, we are then able to use this estimate e.g. in order to predict future allocations. Our contributions are as follows: 

\begin{itemize}
    \item[$(i)$] we propose an algorithm based on matrix factorization enhanced with optimal transport steps to model user-item affinities and learn item embeddings from observed data;
    \item[$(ii)$] we then illustrate and discuss the results of such an approach for hospital recommendation on synthetic data.
\end{itemize}

\subsubsection*{Paper organization} After reviewing related work, we formally define the problem in mathematical terms, we describe our algorithm for Sinkhorn Matrix Factorization with Capacity Constraints (SiMCa) and give theoretical guarantees on its convergence. We then illustrate our method for the hospital recommendation problem on synthetic data, discussing the results as well as the choice of parameters.


\section{Related work}
Hospital and practitioner recommendation has already been studied in the literature (see e.g. the survey \cite{tran_recommender_2021}). However, to the best of our knowledge, no existing method incorporates hospital capacity constraints in the algorithm training. This tends to refer many users to the same hospital, potentially saturating it and degrading the overall care quality. 

Matrix factorization \cite{koren_matrix_2009} is among the most popular collaborative filtering recommendation algorithms. Matrix factorization characterizes every user $i$ and item $j$ by high-dimensional embeddings $u_i, v_j$, and predict the user-item affinity by the inner product $\langle u_i, v_j \rangle$. This method has already been applied for patient/doctor recommendation \cite{zhang_idoctor_2017, han_hybrid_2018}. However, regular matrix factorization is usually applied to simple recommendation problems, such as movie recommendation: as already explained before, recommending locations brings new challenges and requires a different approach \cite{zhao_survey_2016}. 

Geographical influence has been integrated in the matrix factorization framework to recommend locations or points of interest (POIs) \cite{li_rank-geofm_2015}: moreover, the learning algorithm can be adapted by adding a capacity term in the loss function \cite{christakopoulou_recommendation_2017}.

The Monge-Kantorovitch formulation of the classical Optimal Transport (OT) problem can be rephrased as a linear program that can be computationally slow and unstable in high dimension \cite{cuturi_sinkhorn_2013}: this problem is often approximated by adding an entropy regularization term, and easily solved by Sinkhorn-Knopp's algorithm \cite{cuturi_sinkhorn_2013}. Another important advantage of this regularization is that the solution of the OT problem becomes differentiable with respect to the parameters, which explains why this step is integrated in many learning algorithms \cite{genevay_learning_2017,cuturi_soft-dtw_2018,tai_sinkhorn_2021}.

Most relevant for the present paper is the work from Dupuy, Galichon and Sun \cite{dupuy_estimating_2016}. In this study, the authors address the inverse optimal transport problem, that is, given vectors of characteristics $\bX \in \mathbb{R}^d$ and $\bY \in \mathbb{R}^{d'}$ and the joint distribution of the optimal matching, the problem of recovering the affinity function of the form $\phi(\bX,\bY)=\bX^T \bA \bY$, namely to estimate matrix $\bA$. The authors are in the setting where they observe pairs of embeddings $(\bX_t,\bY_t)$ together with the optimal \emph{regularized matching} $\pi^*$ -- that is the solution to problem \eqref{eq:piSinkhorn} hereafter -- and build an estimator of $\bA$ with low-rank constraints, the objective being to isolate important characteristics that carry the most important weight in the matching procedure between $x$ and $y$. We stress the fact that the setting is different in our study: 
we only observe in our case the embeddings $\bU$ of the users and a distance matrix $\bD$, function $\phi$ is known as well as the \emph{pure matching} $\sigma^*$ -- that is the solution of the linear assignment problem \eqref{eq:sigmastarLAP} hereafter, which differs from $\pi^*$ -- and the aim is to infer item embeddings $\bV$. In other words, we do not seek to reconstruct the affinity matrix, but for the learning of items' positions in the user's embeddings space, these positions acting as reference points, upon which prediction of future allocations can be made. Another difference is that the number of items is typically very small compared to the number of users, which justifies that the items are considered static: we also incorporate \emph{capacity constraints} on the allocation problem. 


\section{Problem definition}

\subsection*{A model for latent and geographical affinity}

The setting of the problem is as follows. Consider $n$ \emph{users} $x_1, \ldots, x_n$ embedded in a latent space $\cX$ identified to $\mathbb{R}^d$, with embeddings given by $\bU_1, \ldots, \bU_n$. Also consider $m$ \emph{items} $y_1,\ldots, y_m$ embedded in $\cX$ with embeddings $\bV_1, \ldots, \bV_m$, with $m \leq n$. To each user $x_i$ we assign a single item $y_j$, according to an \emph{affinity matrix} $\bM \in \mathbb{R}^{n \times m}$ given by
\begin{equation*}\label{eq:def_M}
    \bM_{i,j} := \Phi(\bU_i,\bV_j,\bD_{i,j}),
\end{equation*} where $\bD \in \mathbb{R}^{n \times m}$ is known and may be thought of e.g. as a geographical distance matrix between users and items in the underlying euclidean space, say $\mathbb{R}^2$ (we stress the fact that this space is \emph{not} the embedding space $\cX$). We will denote $\bM = \Phi(\bU,\bV,\bD)$ in the sequel. 

We also work under the following constraints: each item $y_j, j \in [m]$ can be assigned to at most $\bC_j$ users. Where $\bC=(\bC_1,\ldots,\bC_m)$ is \emph{capacity} vector. The \emph{total capacity} is defined by 
\begin{equation*}
    s(\bC) := \sum_{j \in [m]} \bC_j,
\end{equation*} and we will assume $s(\bC) = n$.
We define
\begin{equation*}\label{eq:defSigma}
    \Sigma(n,m,\bC) := \left\lbrace \sigma \in \left\{ 0,1\right\}^{n \times m}, \; {\sigma \mathbf{1}_m}  = { \mathbf{1}_n}, {\sigma^T \mathbf{1}_n} = \bC \right\rbrace.
\end{equation*} In the sequel, $\sigma$ will denote both the assignment and its corresponding matrix representation. The optimal assignment $\sigma^*$ is given by
\begin{equation}\label{eq:sigmastarLAP}
    \sigma^*(\bU,\bV,\bD,\bC) := \argmax_{\sigma \in \Sigma(n,m,\bC)} \Tr\left( \sigma^T \bM \right),
\end{equation}
Note that problem \eqref{eq:sigmastarLAP} is an instance of the \emph{Linear Assignment problem} (LAP).

\subsection*{Goal} Assume that we are given the user embeddings $\bU$, the distance matrix $\bD$, the capacities $\bC$ and the optimal assignment $\sigma^* \in \Sigma(n,m,\bC)$. The goal is to learn the item embeddings $\bV$.

\subsection*{Loss metrics, regularization and relaxation}
We will evaluate the performance of a proposed estimate $\widehat{\bV}$ of $\bV$ through the assignment $\widehat{\pi}$ obtained with $\widehat{\bV}$. To compare $\widehat{\pi}$ with $\sigma^*$, we use the usual \emph{cross entropy loss} defined by
\begin{equation*}
    H(\sigma^*,\widehat{\pi}) := - \sum_{i \in [n]} \log{\widehat{\pi}_{i,\sigma^*(i)}} = - \Tr\left( (\sigma^*)^T (\log \widehat{\pi}) \right).
\end{equation*}



As stated before, from a learning perspective, a main issue is that the solution to problem \eqref{eq:sigmastarLAP} is not differentiable w.r.t. $\bV$, the variable of interest. This issue is solved by a relaxation/regularization procedure \cite{cuturi_sinkhorn_2013}:
\begin{itemize}
    \item since the objective function is linear, we first consider the classical relaxation of \eqref{eq:sigmastarLAP} on the polytope of the convex hull of $\Sigma(n,m,\bC)$, namely on
    \begin{equation*}\label{eq:defPi}
    \Pi(n,m,\bC) := \left\lbrace \pi \in [0,1]^{n \times m}, \; {\pi \mathbf{1}_m}  = { \mathbf{1}_n}, {\pi^T \mathbf{1}_n} = \bC \right\rbrace.
    \end{equation*}
    \item moreover, we regularize the objective function in order to perform (automatic) differentiation: this is made possible by the classical entropy regularization in optimal transport.
\end{itemize}
For a small regularization parameter $\eps>0$, the problem then becomes
\begin{equation}\label{eq:piSinkhorn}
    \pi_{\eps}^*(\bU,\bV,\bD,\bC) := \argmax_{\pi \in \Pi(n,m,\bC)} \left[\Tr\left( \pi^T \bM \right) + \eps H(\pi) \right],
\end{equation} where
\begin{equation}\label{eq:Hpi}
    H(\pi) := - \sum_{1 \leq i,j \leq n} \pi_{i,j} (\log \pi_{i,j}-1).
\end{equation}

It is known in the literature \cite{cuturi_sinkhorn_2013} that the solution $\pi_{\eps}^*$ to the convex optimization problem \eqref{eq:piSinkhorn} can be easily computed with Sinkhorn-Knopp’s algorithm, and has the following form:
\begin{equation}\label{eq:formesinkhorn}
    \left(\pi_{\eps}^* \right)_{i,j} = a_i \exp\left( \frac{1}{\eps} \bM_{i,j} \right) b_j,
\end{equation} where $a$ and $b$ are vectors of $\mathbb{R}_+^n$ and $\mathbb{R}_+^m$.
Note that we are back to our initial problem \eqref{eq:sigmastarLAP} when $\eps=0$.

\subsection*{SiMCa Algorithm}
With this new formulation \eqref{eq:piSinkhorn}, we are now able to design an optimization scheme for our learning problem. In our setting the users embeddings $\bU$, the distance matrix $\bD$ and the capacities $\bC$ are known, only the items embeddings $\bV$ are learned. The overall procedure is summarized in Algorithm \ref{algo_learning_embeddings}. Given the current estimate $\bV_t$ at iteration $t$, we compute the solution $\pi^*_{\eps}(\bV_t)$ to problem \eqref{eq:piSinkhorn}, which in turn is used to compute the gradient in $\bV_t$ of the following loss
\begin{equation}\label{eq:loss_global}
    \loss(\bV_t) := H\left(\sigma^*,\pi^*_{\eps}(\bV_t)\right)
\end{equation} to update our estimate of $\bV$ through a gradient step. The gradient in $\bV$
has actually a simple analytical expression:

\begin{lemma}\label{lemma:gradient}
We have
    \begin{equation}
    \label{eq:gradient}
    \nabla_\bV \loss(\bV) = \frac{1}{\eps} \sum_{1 \leq i,j \leq n} (\pi^*_\eps(\bV) - \sigma^*)_{i,j} \nabla_\bV \bM_{i,j} \, .
\end{equation}
\end{lemma} 

\begin{proof}
A very similar expression for the gradient is derived for the maximum likelihood in \cite{dupuy_estimating_2016}. We straightforwardly adapt their derivation to the cross entropy loss \eqref{eq:loss_global}.
Let us denote
\begin{equation}
\label{eq:defV}
V_\eps(\bM) = \max_{\pi \in \Pi(n,m,\bC)} \left[\Tr\left( \pi^T \bM \right) + \eps H(\pi) \right]
\end{equation}
the optimal value of the regularized OT problem \eqref{eq:piSinkhorn}. As well-known in the OT literature, see Proposition 9.2 of \cite{peyre_computational_2020}, its gradient with respect to the affinity matrix $M$ is given by the optimal coupling
\begin{equation}
\label{eq:gradV}
\frac{\partial}{\partial \bM_{i, j}} V_\eps(\bM) =  (\pi^*_\eps)_{i, j}.
\end{equation}
Our cross-entropy loss \eqref{eq:loss_global} is directly related to the optimal value $V_\eps(\bM)$:
\begin{flalign*}
    \loss = H\left(\sigma^*,\pi^*_{\eps}\right) &
          = - \sum_{i,j} \sigma^*_{i,j} \ln (\pi^*_\eps)_{i,j} \tag*{} \\
          &\stackrel{1}{=} - \sum_{i,j} \sigma^*_{i,j} (\tfrac{1}{\eps} \bM_{i,j} + \ln a_i + \ln b_j) \\
          &\stackrel{2}{=} -  \sum_{i,j} \sigma^*_{i,j} \tfrac{1}{\eps} \bM_{i,j}  - \sum_{i,j} (\pi^*_\eps)_{i,j} (\ln a_i + \ln b_j) \\
          &\stackrel{3}{=} -  \sum_{i,j} \sigma^*_{i,j} \tfrac{1}{\eps} \bM_{i,j}  - \sum_{i,j} (\pi^*_\eps)_{i,j} (\ln (\pi^*_\eps)_{i,j} - \tfrac{1}{\eps} \bM_{i,j}) \\
          &\stackrel{4}{=} -  \sum_{i,j} \sigma^*_{i,j} \tfrac{1}{\eps} \bM_{i,j}  - s(\bC) \\
          &- \sum_{i,j} (\pi^*_\eps)_{i,j} (\ln (\pi^*_\eps)_{i,j} - 1) + \sum_{i,j} (\pi^*_\eps)_{i,j} \tfrac{1}{\eps} \bM_{i,j} \\
          &\stackrel{5}{=} -s(\bC) + \tfrac{1}{\eps} [ 
          \Tr(\pi^{*T}_\eps \bM) + \eps H(\pi^*_\eps) - \Tr(\sigma^{*T}\bM)
          ] \\
          &\stackrel{6}{=} -s(\bC) + \tfrac{1}{\eps} [ 
          V_\eps(\bM) - \Tr(\sigma^{*T}\bM)]. &&
\end{flalign*}

The first and third equalities follow from \eqref{eq:formesinkhorn}, the second and fourth from $\sigma^*, \pi^*_\eps \in \Pi(n,m,\bC)$, the fifth from the definition \eqref{eq:Hpi} of $H(\pi)$ and the sixth from the definition \eqref{eq:defV} of $V_\eps(\bM)$.
Then differentiating with respect to $\bV$ leads to \eqref{eq:gradient} by the chain rule and \eqref{eq:gradV}.
\end{proof}

\begin{algorithm}[h]
\caption{\label{algo_learning_embeddings} Sinkhorn Matrix Factorization with Capacity Constraints (SiMCa)}
\begin{flushleft}
    \textbf{Input:} $\bU, \bD, \bC, \sigma^*$

    For $t =1$ to $T$:
    
    \begin{itemize}
    	\item[$1.$] Compute the affinity matrix $\bM_{t-1} = \Phi(\bU,\bV_{t-1},\bD)$. 
    	
    	\item[$2.$]  Compute the solution to the optimization problem \eqref{eq:piSinkhorn}:
    	\begin{equation*}
    	\pi_{\eps}^*(\bV_{t-1}) := \argmax_{\pi \in \Pi(n,m,\bC)} \left[\Tr\left( \pi^T \bM_{t-1} \right) + \eps H(\pi) \right].
    	\end{equation*}
    	
    	\item[$3.$]  Compute the gradient $\nabla \loss(\bV_{t-1})$ with equation \eqref{eq:gradient}.
    	
    	\item[$4.$]  Perform a gradient step $\bV_{t} = \bV_{t-1} - \eta \nabla \loss(\bV_{t-1})$.
    \end{itemize}
\textbf{return} $\bV_T$
\end{flushleft}
\end{algorithm}

The performance of our method is guaranteed by the following:
\begin{lemma}\label{lemma:convergence_guarantee}
Assume that $v \mapsto \Phi(u,v,d)$ is linear. Then the loss function \eqref{eq:loss_global} is convex in $\bV$ and the output of SiMCa Algorithm (Algo. \ref{algo_learning_embeddings}) converges to 
\begin{equation*}
    \argmin_{\bV} H\left(\sigma^*,\pi^*_{\eps}(\bV)\right).
\end{equation*}
\end{lemma}

\begin{proof}
The proof of Lemma \ref{lemma:gradient} shows that
\begin{equation*}
    \loss(V) = -s(\bC) + \tfrac{1}{\eps} [ 
          V_\eps(\bM) - \Tr(\sigma^{*T}\bM)].
\end{equation*}
Since $V \mapsto \Phi(\bU,V,\bD)$ is linear, $V \mapsto V_\eps(\bM)$, as defined in \eqref{eq:defV} is convex as a maximum of convex functions. By assumption, $V \mapsto \Tr(\sigma^{*T}\bM)$ is linear, thus $V \mapsto \loss(V)$ is convex.
\end{proof}


\section{Illustration for the hospital recommendation problem}
We now describe an illustration of our method for the hospital recommendation problem. Since very few open datasets are available for this problem, we trained our algorithm on synthetic data. 

\subsection*{Dataset generation}
The dataset is generated as follows:
\begin{itemize}
    \item \textbf{Features in the embedding (latent) space:} we sample $n+m$ points from a Gaussian mixture model with $k$ clusters. We set these points as either users ($\bU_i$) or items ($\bV_i$), and considered that each cluster must contain at least one item: we are thus left with $n$ users and $m$ items, spread between $k$ clusters. Users and items in the same cluster are considered similar. We then normalized both users and items features, so that all embeddings $\bU_i$ and $\bV_j$ lie on the unit sphere. Note that the users and items sampling is done independently of items capacities. 
    
    \item \textbf{Distance in the underlying euclidean space:} to sample the distance matrix $\bD$ between users and items, we sample all the positions randomly on a circle, and computed the great-circle distance (i.e. spherical distance) between every users $i$ and items $j$. We finally normalize the distance matrix by its overall mean.
    
    \item \textbf{Capacities} we sampled $m$ values from a Dirichlet Distribution, corresponding to the probabilities that users are assigned to the $m$ items. We converted these probabilities into capacities $\bC_j$ by multiplying them with the number of users $n$. We then added some extra spots to each item. 
\end{itemize}

\subsubsection*{Affinity matrix}
In our case, the affinity matrix $\bM=\Phi(\bU,\bV,\bD)$ is defined as follows:
\begin{equation}\label{eq:phi_def_illustration}
    \bM_{i,j}=\Phi(\bU_i,\bV_j,\bD_{i,j}) = (1-\alpha) \bU_i^T \bV_j - \alpha \bD_{i,j}.
\end{equation} The $\alpha$ coefficient measures the trade-off between affinity and proximity.

We then solve the Linear Assignment Problem \eqref{eq:sigmastarLAP} to compute the pure matching $\sigma^*$.

\subsubsection*{Noise}
Noise is added to the original dataset in two different ways. The first method is to modify the allocations of random users in $\sigma^*$, the noise ratio being defined as the percentage of modified allocations\footnote{to make sure that the capacities constraints on the items still hold, we must swap \emph{pairs of users}: for a given allocation to modify, we pick another user randomly and swap their allocations.}. The second method consists in perturbating $\bU$ as follows:
\begin{equation*}
    \widetilde{\bU} := \sqrt{1-\rho^2} \bU + \rho \bZ,
\end{equation*} where $\bZ$ is a matrix with i.i.d. standard Gaussian entries, and $\rho$ is the noise ratio. 

\subsubsection*{Learning the embeddings}
Given $\bU$, $\bD$, $\bC$, $\sigma^*$, $\alpha$ and $\eps$, we compute an estimate $\widehat{\bV}$ of the item embeddings with SiMCa Algorithm (Algo. \ref{algo_learning_embeddings}). Comparing $\widehat{\bV}$ with $\bV$ gives a first measure of the training performance.

\subsubsection*{Recovering the pure matching}
Then, using $\widetilde{\bU}$ (the noisy version of $\bU$), $\widehat{\bV}$ (the estimated $\bV$), $\bD$, $\alpha$ and $\eps$, we compute the solution $\widehat{\pi}^*_{\eps}$ to problem \eqref{eq:piSinkhorn}. Solving the linear assignment problem (LAP) on matrix $\widehat{\pi}^*_{\eps}$, we compute a pure matching $\widehat{\sigma}^*$, which we can next compare to the original ${\sigma}^*$, giving a second measure of the training performance.


\section{Results}

\subsubsection*{Parameters}
We generated a toy dataset with the following parameters: $n=1000$ users; $m=3$ items; $d=2$ latent features; $k=3$ clusters; $\alpha=0.3$. The items capacities were 257, 417 and 356. Figure \ref{fig:toy_dataset} shows the generated users and items in both the embeddings (latent) space and their underlying euclidean space.

\begin{figure}[h]
    \centering
    \includegraphics[width=.9\columnwidth]{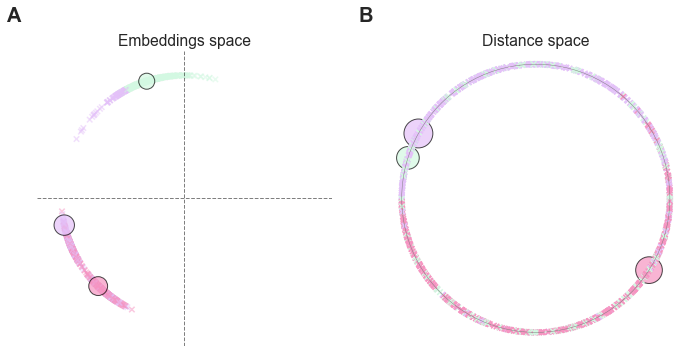}
    \caption{Generated dataset. Users are displayed as crosses and items as circles, sized proportionally to their capacities. Users are colored accordingly to the item they have been allocated to. Plot (A) displays users and items in their shared embeddings (latent) space; where plot (B) displays them in their underlying euclidean space.}
    \label{fig:toy_dataset}
\end{figure}

\subsubsection*{Training results}
We trained our model with $\eps=0.1$ entropy regularization and $10$ iterations in Sinkhorn-Knopp’s algorithm to output $\widehat{\bV}$. As mentioned earlier we compute a solution to the LAP on matrix $\widehat{\pi}^*_{\eps}$ to output the estimated allocation $\widehat{\sigma}^*$. The model was trained with Adam optimizer, with a $0.01$ learning rate and $400$ epochs. For the measures of performance, we used the F1 score, for measuring how well the allocations are reproduced, and the mean euclidean distance between learned embeddings $\widehat{\bV}$ and the ground truth $\bV$. The training results are displayed on Figure \ref{fig:learned_embeddings}.

\begin{figure}[h]
    \centering
    \includegraphics[width=.9\columnwidth]{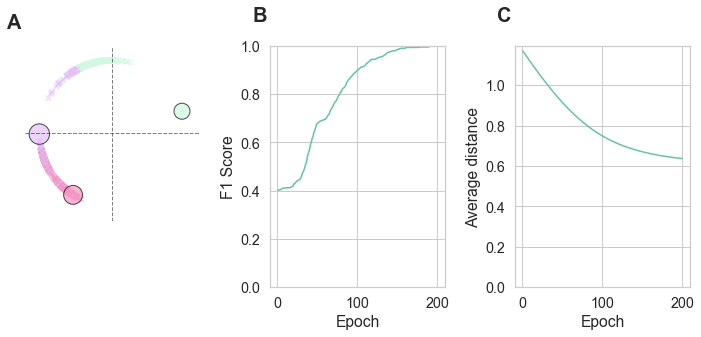}
    \caption{
        Training results. We can see that the model achieves good performances to learn the item embeddings (A), and recovers the allocation with a close to 1 F1 score (B). The average distance between the learned embeddings and ground truth decreases during training (C).
    }
    \label{fig:learned_embeddings}
\end{figure}

\subsubsection*{Influence of entropy regularization}
We investigate the influence of the entropy regularization parameter $\eps$ on the model performance. We let $\eps$ vary between $0.05$ and $2$, with the same dataset and the same model parameters\footnote{due to numerical instability, the algorithm could not train properly above $\eps = 0.05$.}. We ran 5 training for every value of $\eps$. As shown on Figure \ref{fig:experiment_epsilon}, the training performance worsens when $\eps$ increases. 

\begin{figure}[h]
    \centering
    \includegraphics[width=.45\columnwidth]{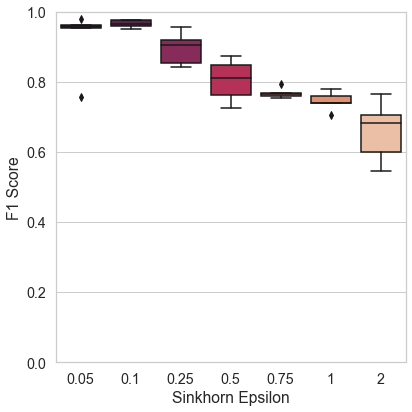}
    \caption{
        Performance as a function of $\eps$. Increasing $\eps$ leads to lower F1 scores.
    }
    \label{fig:experiment_epsilon}
\end{figure}

\subsubsection*{Influence of noise}
We study the influence of noise, either by swapping allocations or adding Gaussian noise to the used embeddings, as described in the previous section. Unsurprisingly, as shown in Figure \ref{fig:experiment_noise}, the training performance is decreasing with the noise ratio.

\begin{figure}[h]
    \centering
    \includegraphics[width=.9\columnwidth]{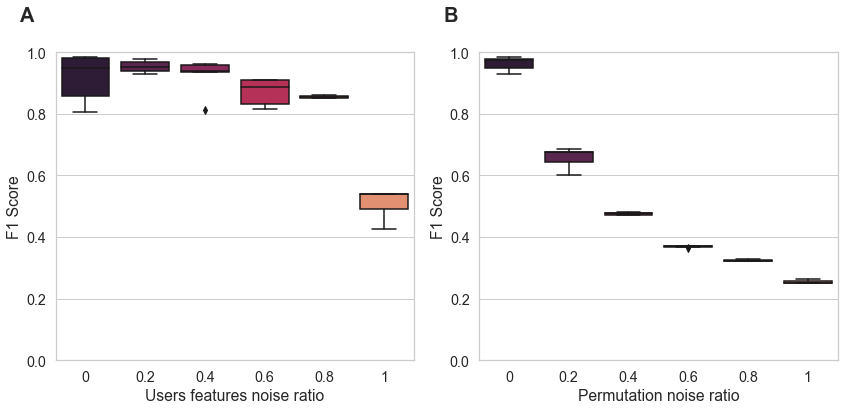}
    \caption{
        Influence of adding Gaussian noise (A) and swapping allocations (B) on training performance. F1 score decreases as noise increases.
    }
    \label{fig:experiment_noise}
\end{figure}

\subsubsection*{Learning both users and items embeddings}

We also studied the case where the users' embeddings are not known, and must be learned jointly with the items' embeddings from the observed allocations. In this case, we initialized the items' and users' embeddings similarly. We managed to retrieve the observed allocation as illustrated on Figure \ref{fig:experiment_learn_users_embeddings}. However, the average distance between learned items embeddings and ground truth does not decrease during training, meaning that the model learned its own interpretation of the users' and items' representations to satisfy the observed mapping. 

\begin{figure}[h]
    \centering
    \includegraphics[width=.9\columnwidth]{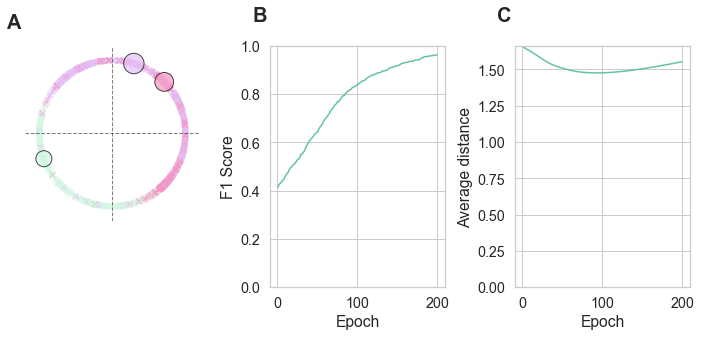}
    \caption{
        Learning both users and items embeddings simultaneously. The learned embeddings are shown on (A). The model retrieves the observed allocation (B). However, the average distance between learned items embeddings and ground truth does not decrease during training (C).
    }
    \label{fig:experiment_learn_users_embeddings}
\end{figure}


\section{Conclusions}

In this work, we introduced SiMCa, an algorithm based on matrix factorization and optimal transport to model user-item affinities and learn item embeddings from observed data. 
SiMCa can be used in recommendation problems where allocations between users and items are based on: their affinity in a latent space; a geographical distance in their underlying euclidean space; capacity constraints on the items. We illustrated our method for the hospital recommendation task; however, we believe that there are many other problems for which SiMCa algorithm may be useful. 


\section*{Acknowledgments}
This work was partially supported by the French government under management of Agence Nationale de la Recherche as part of the “Investissements d’avenir” program, reference ANR19-P3IA-0001 (PRAIRIE 3IA Institute).


\bibliography{main.bib}
\bibliographystyle{plain}

\end{document}